\documentclass[twoside]{article}

\usepackage[round]{natbib}

\usepackage[accepted]{aistats2026}

\usepackage[T1]{fontenc}% added to have bold smallcaps

\usepackage{microtype}
\usepackage{graphicx}
\usepackage{booktabs} % for professional tables
\usepackage{makecell}
\usepackage{bigdelim}

\usepackage[bookmarksnumbered,debug,backref=page]{hyperref}

\usepackage{wrapfig}
\usepackage{adjustbox}
\usepackage{subcaption} % Permet d’insérer au sein d’une figure des sous-figures, chacune d’entre elles disposant d’une légende, en plus de la légende principale.

\makeatletter
\newcommand*{\centerfloat}{%
  \parindent \z@
  \leftskip \z@ \@plus 1fil \@minus \textwidth
  \rightskip\leftskip
  \parfillskip \z@skip}
\makeatother

\usepackage{algorithm}
\usepackage{algorithmic}

\usepackage{amsmath}
\usepackage{amssymb}
\usepackage{mathtools}
\usepackage{amsthm}
\usepackage{dsfont}
\usepackage{thmtools} 
\usepackage{thm-restate}

\usepackage[capitalize,nameinlink]{cleveref}

\theoremstyle{plain}

\theoremstyle{definition}

\theoremstyle{remark}

\usepackage[textsize=tiny]{todonotes}

\usepackage[maxlevel=3]{csquotes}
\usepackage{siunitx}
\usepackage{shortcuts}

\usepackage{xcolor}         % colors
\usepackage[normalem]{ulem}

\crefname{ALC@unique}{Line}{Lines}
\Crefname{ALC@unique}{Line}{Lines}
\crefname{algorithm}{Alg.}{Algs.}
\crefname{table}{Tab.}{Tabs.}
\crefname{section}{Sec.}{Secs.}
\crefname{theorem}{Thm}{Thms}
\crefname{lemma}{Lem.}{Lems.}
\crefname{appendix}{App.}{Apps.}

\newcommand{\rosecdl}{{\textsc{RoseCDL}}}% Changed from \sc to this to enable bold sc

\definecolor{linkcolor}{RGB}{83,83,182}

\hypersetup
{
  colorlinks=true,
  citecolor=linkcolor,
  linkcolor=linkcolor,%black,
  urlcolor=linkcolor,
  pdfstartview=FitH,
  pdftitle={},
  pdfauthor={},
  pdfcreator={},
  pdfsubject={},
  pdfkeywords={},
}

\begin{document}
\runningauthor{Jad Yehya, Mansour Benbakoura, Cédric Allain, Benoît Malézieux, Matthieu Kowalski, Thomas Moreau}

\runningtitle{\rosecdl{}: Robust and Scalable Convolutional Dictionary Learning for Rare event and Anomaly Detection}

\twocolumn[
    \aistatstitle{\rosecdl{}: Robust and Scalable Convolutional Dictionary Learning for Rare-event and Anomaly Detection}

\aistatsauthor{%
  % examples of more authors
Jad Yehya$^{1,*}$ \And Mansour Benbakoura$^{1,*}$ \And Cédric Allain$^1$ \And Benoît Malézieux$^1$ \AND Matthieu Kowalski$^2$ \And Thomas Moreau$^1$
}
\vspace{6pt}
\aistatsaddress{$^1$Université Paris-Saclay, Inria, CEA, Palaiseau, 91120, France.\\
$^2$Université Paris-Saclay, Inria, CNRS, Laboratoire Interdisciplinaire\\ des Sciences du Numériques, Gif-sur-Yvette, France.\\$^*$Equal Contributions.}
]
\begin{abstract}
\looseness=-1
Detecting rare events and anomalies in large-scale signals is essential in fields such as astronomy, physical simulations, and biomedical science.
In many cases, this problem naturally decomposes into identifying common local patterns and detecting deviations that correspond to anomalies.
Convolutional Dictionary Learning (CDL) is a powerful tool for modeling local structures, but its adoption for this task has been limited by computational demands and sensitivity to outliers.
We introduce \rosecdl{}, a novel CDL algorithm designed for robust and scalable modeling  of signal pattern distribution.
\rosecdl{} leverages stochastic windowing for efficient training and incorporates inline outlier detection to enhance robustness.
This enables unsupervised identification of anomalous and rare patterns in long signals based on the local reconstruction loss.
Experiments on real-world datasets show that \rosecdl{} delivers improved detection accuracy and computational efficiency, making CDL practical for challenging detection tasks in large-scale signal analysis.
\end{abstract}
\section{INTRODUCTION}

Identifying recurring patterns and anomalies in a signal is a crucial task in many scientific fields, from finding QRS complex -- \emph{a.k.a.} heartbeats -- in ECG~\citep{Luz2016} to detecting blood-cells in biological images~\citep{Yellin2017} or particular celestial objects in astronomical images~\citep{Giavalisco2004}.
In large-scale settings, this process must be automated.
Supervised methods have been developed to address these tasks, often relying on large annotated datasets and deep learning models~\citep{Murat2020, Choudhary2024, Cornu2024}.
However, obtaining labeled data can be costly, time-consuming, or even infeasible, especially when the patterns of interest are rare or hard to characterize.
This motivates the development of unsupervised methods capable of automatically uncovering patterns distribution in large datasets.

Convolutional Dictionary Learning~\citep[CDL;][]{Grosse2007} is a powerful method for modeling local structures in signals, with applications in audio, neuroscience, and image processing \citep{dupre2018multivariate,Papyan_2017_ICCV}.
Its core principle is to represent an observed signal as the convolution of a learned dictionary of patterns with a sparse activation vector.
By aggregating over many signals, CDL learns a dictionary that captures recurring local structures, making it particularly suitable for unsupervised pattern discovery.
Despite its promise, the use of CDL in practice remains limited.
Two challenges, in particular, hinder its applicability to large, noisy datasets.

\paragraph{Scalability.}
Learning with CDL is computationally expensive, primarily due to the sparse coding phase.
Considerable effort has been devoted to improving scalability, including deterministic optimizations~\citep{wohlberg2015efficient}, local-window methods that partition the signal~\citep{Grosse2007,Papyan_2017_ICCV, moreau2020dicodile, dragoni_2022}, and approximate techniques such as learned sparse coding and algorithm unrolling~\citep{gregor_lecun,Tang2021,tolooshams_2018,tolooshams2022stable,malezieux2021understanding}.
Online algorithms~\citep{mairal2010online,mensch2016dictionary,Liu2017,Zeng2019} further reduce computational cost by updating the dictionary on subsets of the data.
However, existing approaches either struggle to generalize across large datasets or achieve only limited computational gains.

\paragraph{Robustness to anomalies and rare events.}
CDL is highly sensitive to anomalies: because outliers are sparse but often high-magnitude, the algorithm may mistakenly interpret them as meaningful patterns.
Theoretical work in classical dictionary learning has formalized this risk~\citep{gribonval_sparse_spurious}, and a few practical remedies have been proposed, such as Elastic Net regularization~\citep{mairal2010online} or heavy-tailed fitting terms~\citep{jas2017learning}.
Yet these methods either increase computational cost or only partially address robustness.
Related research in outlier detection (OD) has emphasized reconstruction-based frameworks, where models are trained to reconstruct normal data while failing on anomalies~\citep{Ruff2021, Schmidl2022}. 
While the majority of reconstruction-based OD methods are semi-supervised, fully unsupervised methods perform competitively~\citep{darban2024deep}, and similar principles are used for rare-event detection~\citep{shyalika2024}.
This idea also resonates with classical robust regression, where corrupted observations are trimmed based on residual magnitude.
The Least Trimmed Squares (LTS) estimator~\citep{rousseeuw1984least,rousseeuw2005robust} and its extensions, including Sparse LTS~\citep{alfons2013sparse} demonstrate the effectiveness of trimming in high-dimensional and structured settings.

\paragraph{Our contributions.}
This paper introduces a novel CDL algorithm called RObust and ScalablE CDL (\rosecdl{}) that addresses these limitations.
To ensure scalability, we propose a \textbf{stochastic windowing approach}, where the sparse coding problem is solved approximately on a small, randomly sampled data windows, and the dictionary is updated using the results of local computations.
To ensure robustness, we integrate an \textbf{inline outlier detection} mechanism based on local reconstruction error, which discards patches poorly explained by the learned dictionary.
The same reconstruction-error criterion can be applied at test time to identify anomalies and rare events, providing a practical unsupervised detection tool alongside representation learning.
The resulting algorithm learns dictionaries that capture the local structure of data, remains robust to artifacts, scales to large datasets, and naturally doubles as an unsupervised anomaly detection method.
More broadly, it \textbf{reframes CDL from a reconstruction problem to the estimation of the patch distribution underlying a signal}, enabling both robust representation and principled rare-event and anomaly detection.
We demonstrate its effectiveness on both synthetic and real-world data.

\section{FINDING COMMON AND RARE PATTERNS IN SIGNALS: THE \rosecdl{} ALGORITHM}
\label{sec:methods}
Let $\bx \in \R^{T}$ be a univariate signal with length $T$.
Convolutional Dictionary Learning (CDL) consists of finding a dictionary: $D = (\bd_k)_{k \in \llbracket 1, K \rrbracket} \in \R^{K \times L}$ of $K$ patterns of length $L \ll T$, and corresponding activation vectors $Z = (\bz_k)_{k \in \llbracket 1, K \rrbracket} \in \R^{K \times (T-L+1)}$, that minimize the distance between $\bx$ and $\hat{\bx} = D * Z = \sum_k \bd_k * \bz_k$, where $*$ denotes the convolution.

\begin{figure*}[t]
    \centering
    \begin{adjustbox}{center}
    \includegraphics[width=0.8\textwidth]{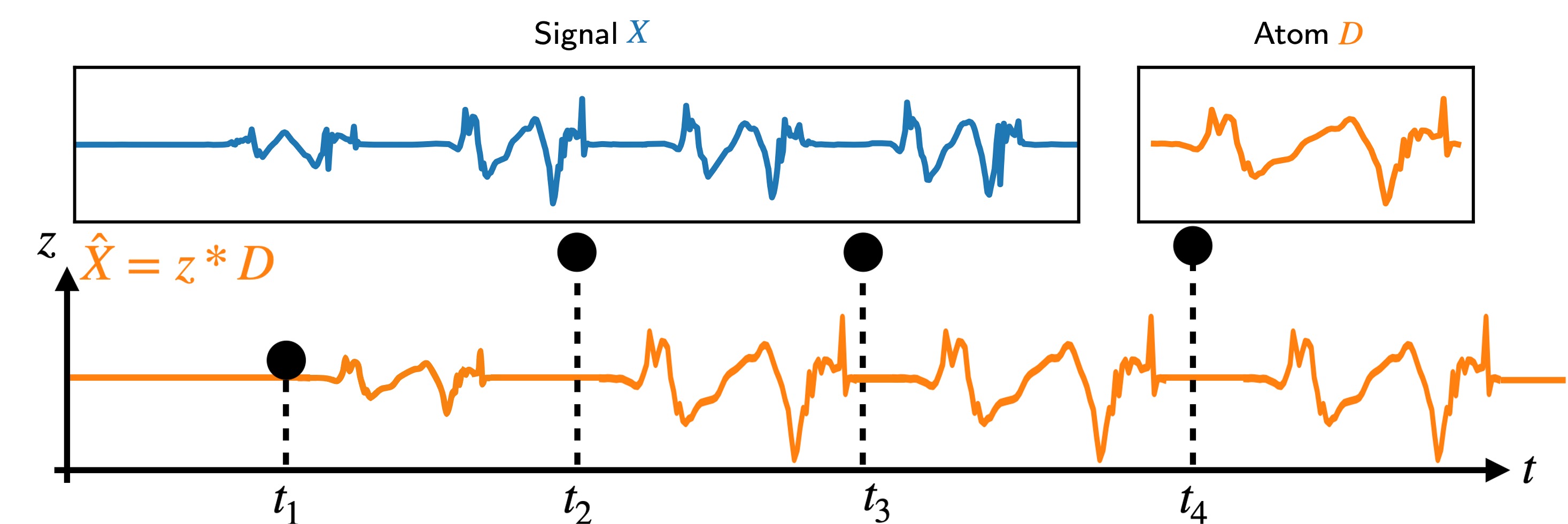}
    \end{adjustbox}
    \caption[Schematic operation of the CDL.]{
        Schematic operation of the CDL 1D univariate signal, adapated from the \texttt{dicodile} package example\protect\footnotemark.
        The output of the CDL is composed of a set of atoms (here 1) alongside their respective activations.
    }
    \label{fig:cdl_resume}
\end{figure*}

This is achieved by solving the following optimization problem:
% \begin{equation}
%     \label{eqn:cdl}
%         \min_{D, Z} \quad F(D, Z; \bx) = \underbrace{\frac{1}{2} \|\bx - D * Z\|_2^2}_{f_\bx(Z)} + \lambda \|Z\|_1,
%         \quad \text{s.t.} \quad \|\bd_k\|_2^2 \le 1,\ \forall k \in \llbracket 1, K \rrbracket,
% \end{equation}

\begin{equation}
\label{eqn:cdl}
\begin{aligned}
    \min_{D, Z} \quad & F(D, Z; \bx) 
        = \underbrace{\tfrac{1}{2} \|\bx - D * Z\|_2^2}_{f_\bx(Z)} 
        + \lambda \|Z\|_1, \\
    \text{s.t.} \quad & \|\bd_k\|_2^2 \le 1, \quad 
    \forall k \in \llbracket 1, K \rrbracket.
\end{aligned}
\end{equation}

where $\| Z \|_1 = \sum_k \| \bz_k \|_1$.
\cref{fig:cdl_resume} illustrates the CDL model for an univariate signal.
This problem can easily be extended to multivariate and multidimensional signals such as images by adapting the convolution operator to work with tensors.

When working with large signal databases, the primary interest is not finding the best reconstruction of the signal used to learn the dictionary but instead finding the patterns that best model the population.
Therefore, this paper focuses on \textbf{characterizing the distribution} of the signals $\bx$, with the following optimization problem:
\begin{equation}
    \label{eqn:expected_cdl}
    \begin{aligned}
        \min_D \quad & \bbE_{\bx} \left[ \min_Z F(D, Z; \bx) \right], \\
        \text{s.t.} \quad & \|\bd_k\|_2^2 \le 1,\ \forall k \in \llbracket 1, K \rrbracket.
    \end{aligned}
\end{equation}
While this formulation departs from classical optimization-based literature for CDL~\citep{wohlberg2015efficient}, it is often used with deep CDL approaches to learn denoisers that generalize to unseen images using super learning losses~\citep{scetbon2021deep,zheng2021deep,deng2023deep}.
Here, we consider the unsupervised loss $F$ to learn the dictionary of event patterns directly from the signals.

For both formulation~\eqref{eqn:cdl} and~\eqref{eqn:expected_cdl}, classical solvers rely on an iterative optimization algorithm to solve the sparse coding problem in
$Z$ for fixed $D$, and alternate it with a dictionary update step.
The sparse coding step is typically solved using iterated soft-thresholding iterations~\citep[e.g., FISTA,][]{beck_fast_2009,chambolle2015convergence}, which reads:
\begin{equation}\label{eqn:fista}
\begin{split}
       t_0 &= 1\\
       Z_0 &= Y_0\in \bbR^{K\times T - L +1};\ t_0 = 1\\
       Z_{k+1} &= \mathrm{St}(Y_k - \nu \nabla_Z f_{\bx}(Z), \nu\lambda)\\
       t_{k+1} &= \frac12(1 + \sqrt{1 + 4t_k^2}) \\
       Y_{k+1} &= Z_k + \frac{t_k -1}{t_{k+1}}(Z_{k+1} - Z_k)
\end{split}
\end{equation}
with $\mathrm{St}$ the soft-thresholding operator $\mathrm{St}(Z) = \mathrm{sign}(Z)(|Z| - \lambda)_+$ and $\nu$ the step size.
The dictionary update may rely on projected gradient descent constrained to the unit $\ell_2$-ball.
However, these methods become computationally expensive on long signals or large datasets, as both the sparse codes and gradients must be computed over all the signals. Online methods mitigate the need to consider all signals, but at the expense of increased memory usage and still require solving the sparse coding problem on full signals. This motivates using stochastic or localized approximations, especially in settings where robustness to artifacts or outliers is required.

\subsection{Stochastic windowing}
Due to the convolutional structure of the CDL model, points in the far-apart signal are only weakly dependent~\cite{moreau2020dicodile}.
Indeed, the value of the sparse code at time $t$ is seldom impacted by the values at time $t + s$, with $s$ larger than the size of the dictionary $L$.
We will therefore perform windowing spanning across all channels synchronously.

\footnotetext{\url{https://tommoral.github.io/dicodile/auto_examples/plot_gait.html}}

By considering long enough windows of the signal, the problem~\eqref{eqn:expected_cdl} can be approximated by
\begin{equation}
    \label{eqn:patch_cdl}
    \begin{aligned}
        \min_D \quad & \bbE_{\tau} \left[ \min_Z F(D, Z; \bx_{\tau}) \right], \\
        \text{s.t.} \quad &  \|\bd_k\|_2^2 \le 1,\ \forall k \in \llbracket 1, K \rrbracket,
    \end{aligned}
\end{equation}
where $\bx_\tau$ are windows of the signal, starting at $\tau \in \llbracket 1, T - W_{\mathrm{win}} + 1 \rrbracket$ and of size $W_{\mathrm{win}}$ such that $L \le W_{\mathrm{win}} \ll T$.
This formulation is only approximate as it produces inexact sparse code for the windows, due to border effect.
But with the weak spatial dependence of the model, if the window is large and the activation are sparse, these effects are negligible.
Moreover, by considering all windows with overlap, we limit the impact of specific border effects in the algorithm.
% Instead of solving the optimization problem \cref{eqn:cdl} for the whole signal, we split it into windows and solve the resulting subproblems in parallel.
% We introduce two parameters: the size of the windows $W_{\mathrm{win}}$ and the number of windows $N_W$.
To minimize~\eqref{eqn:patch_cdl}, we propose \rosecdl{}, a stochastic gradient descent algorithm aimed to characterize the distribution of patches in the signal.

We process as follows: At each step of the CDL outer problem, we begin by sampling $N_W$ windows $(\bx_w)_{1 \le w \le {N_W}}$ from a uniform distribution.
The windows are selected with overlap to limit the bias due to border effects.
We then compute an approximate sparse code $\Zstar_w$ for each subproblem by minimizing the sparse coding over the window.
Following \cite{malezieux2021understanding,tolooshams2022stable}, whose results suggested that optimizing over $D$ did not require precise sparse coding at each step, we use an approximation $Z^{\Nfista}$ of $\Zstar(D; \bx)$, where $Z^{\Nfista}$ is given by $\Nfista$ iterations of the FISTA algorithm~\eqref{eqn:fista}.

Using these per-window approximated activation vectors, we can then update the dictionary $D$.
Due to the stochastic nature of the algorithm, we depart from traditional alternate minimization strategies and perform only one gradient step.
This single-step update strategy is further justified by the inherent noise in the gradient estimate, which arises from both the stochastic sampling of windows and the use of approximate sparse codes. In such settings, performing multiple gradient steps per batch does not significantly reduce the update variance and may even amplify the impact of the approximation error in the sparse coding stage.
This strategy is in-line with the deep CDL algorithm, but we do not backpropagate the gradient through the sparse code approximation $Z^{\Nfista}$, as it leads to unstable Jacobian estimation for the original problem~\eqref{eqn:patch_cdl}, as described in \cite{malezieux2021understanding}.
% Then, we compute the loss and its gradient with respect to $D$, excluding these abnormal segments.
To stabilize the learning and reduce the number of parameters, we compute the optimal stepsize of the dictionary update with the Stochastic Line Search algorithm \citep[SLS;][]{vaswani2019painless}.
This algorithm can be efficiently implemented using deep learning frameworks and can leverage GPU acceleration.
The complete procedure is summarized in \cref{alg:windowed_cdl}.
We note that when trimming is ignored and the sparse codes are solved exactly, \rosecdl{} fits within the SGD setting, with stochasticity arising from the sampled windows. In this setting, convergence properties are well established~\citep{bottou2018optimization}

\begin{algorithm}[t]
    \caption{CDL with stochastic windowing.}
    \begin{algorithmic}
        \INPUT X, $N_{\mathrm{iter}}$, $N_W$, $\Nfista$
        \STATE Initialize $D^{(0)}$
        \FOR{$0 \le i \le N_{\mathrm{iter}} - 1$}
        \STATE Sample $N_W$ windows in the dataset: $(X_w)_{w \in \llbracket 1, N_w \rrbracket}$
        \FOR{$1 \le w \le N_W$}
        \STATE Compute the approximate sparse code\\ $Z^{\Nfista}_w \approx \Zstar_{w}(D^{(i)}; X_w)$
        \STATE Compute an outlier mask (cf. \cref{subsec:outlier_detection})
        \STATE Compute the loss $F$ and its gradient $\nabla_{D} F$ outside the outlier mask
        \ENDFOR
        \STATE Compute best step size $\alpha_i$ with SLS
        \STATE $D^{(i+1)} \leftarrow D^{(i)} - \alpha_i \nabla_D \sum_w F_w (D^{(i)}, Z^{\Nfista}_w; X_w)$
        \ENDFOR
        \OUTPUT $D^{(N_{\mathrm{iter}})}$
    \end{algorithmic}
    \label{alg:windowed_cdl}
\end{algorithm}

\subsection{Inline outlier detection}
\label{subsec:outlier_detection}
The second component of \rosecdl{} is its inline outlier detection framework.
Let $\bx$ be a signal of the form:
\begin{equation}
\label{eqn:three_regimes}
\bx = \bd_a * \bz_a + \bd_b * \bz_b + \bn,
\end{equation}
where $\bd_a$ is a common pattern in the signal, $\bd_b$ is a rare pattern, and $\bn$ is an abnormal pollution (e.g., artifacts, sudden spikes in the signal).
Traditional CDL algorithms are likely to struggle to recover $\bd_a$ for two main reasons.
First, artifacts in $\bn$ often have a significant variance, which can distract the CDL from the significant signal.
Second, even when the anomalies in $\bn$ are well discarded, the rare pattern $\bd_b$ acts as a pollution that prevents the algorithm from learning the atom $\bd_a$.
While effects of $\bn$ and $\bd_b$ are often discarded through preprocessing~\cite{dupre2018multivariate}, it is often very difficult to use this reliably on large population datasets.

The intuition behind our approach is that if $\bd_a$ is sufficiently represented in $\bx$, then most of the patches of $\bx$ should contain information relative to this pattern.
Consequently, these patches should be the best reconstructed ones.
In comparison, the patches containing non-zero values of $\bn$ are expected to have a high reconstruction error due to the unpredictable nature of anomalies and artifacts.
Finally, given a dictionary $\bd$ that is more correlated with $\bd_a$ than $\bd_b$, the patches containing chunks of $\bd_b$ are expected to have a higher reconstruction error than those with $\bd_a$.

\begin{figure*}[ht]
\centering
\includegraphics[width=\linewidth]{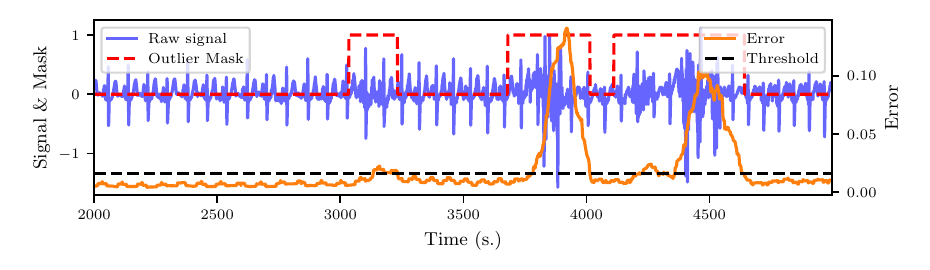}
\vskip-1em
\caption{Raw signal $X$, reconstruction error, threshold and learned outlier mask on subject a02 (minute 56) of Physionet Apnea-ECG data set. Detection method is based on modified z-score (MAD), with $\alpha=3.5$. The method correctly identifies outliers blocks.}
\label{fig:apnea_outliers_detection_viz}
\end{figure*}

To summarize, the distribution of patch reconstruction errors $F(D, Z; \bx_{\tau})_{\tau}$ is expected to have three modes:
\begin{enumerate}
\item The main, low-value mode corresponding to chunks of $\bx_a = \bd_a * \bz_a$,
\item A secondary mode with a slightly higher value corresponding to chunks of $\bx_b = \bd_b * \bz_b$,
\item A high-value mode corresponding to chunks of $\bn$.
\end{enumerate}
Consequently, the reconstruction error of a patch can be used as an indicator to determine whether it contains information about the relevant signal.
We leverage this intuition to inflect the learning trajectory of $\bd$ towards $\bd_a$.
We define the set $\mathcal{P}_{\beta} = \{\tau | F(D, Z; \bx_{\tau}) < \beta\}$, where $\beta \in [0, 1]$ is a threshold selecting the proportion of outliers.
With a well chosen $\beta$, $\mathcal{P}_{\beta}$ indicates the patches that coincide with realizations of $\bx_a$.
Therefore, to make the CDL converge towards $\bd_a$, we change the objective value for the dictionary updates for its trimmed version:
\begin{equation}
\widetilde{F}(Z,D;\bx) = \frac{1}{W_{\mathrm{patch}}} \sum_{\tau \in \mathcal{P}_{\beta}} F(Z, D; \bx_{\tau}).
\end{equation}

Using this trimmed loss, we can show that in simple settings, the \rosecdl{} algorithm is more robust to the presence of outliers.

\begin{restatable}[Stability of the common pattern]{proposition}{stability}
Consider a population of signals $X$ composed of two patterns $\bd_a$ and $\bd_b$, with activations such that the patterns do not overlap in the signal, and corrupted by an additive Gaussian noise $\epsilon \sim \mathcal N(0, \sigma Id)$.
Introduce $c = \bd_a^\top \bd_b$ and $\rho$ the proportion of rare-event pattern $\bd_b$ activations in the population.
Then, in the noiseless setting,
\begin{enumerate}\renewcommand{\theenumi}{\roman{enumi}}
\item $\bd_a$ is a fixed point of the classical CDL algorithm for $K = 1$ if $c \le \lambda$
\item $\bd_b$ is a fixed point of \rosecdl{} algorithm for $K=1$ if $c \le \lambda$ or the \rosecdl{} algorithm is used with an outlier threshold trimming a proportion of windows greater than $\rho$.
\end{enumerate}
\end{restatable}
The proof is deferred in \Cref{sec:study}.
This theoretical analysis is intentionally illustrative: its purpose is to provide intuition on why trimming improves robustness, rather than to serve as a full convergence guarantee for the algorithm. Nevertheless, this demonstrates that even in ideal conditions (no overlap and starting with the right pattern), classical CDL fails to recover the common pattern due to interference from rare events, while \rosecdl{} corrects this via trimming.
However, a critical design choice is the selection of the statistic $\beta$ used as the threshold, discussed below.

\begin{figure*}[!ht]
    \centering
    \includegraphics[width=\textwidth]{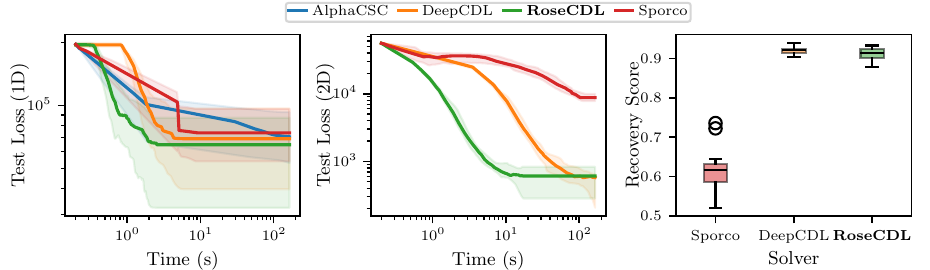}
    \caption[\rosecdl{} efficiency on simulated and real data]{
    Comparison of optimization runtime for \textbf{\rosecdl{}}, \alphacsc, Sporco, and DeepCDL in 1D and 2D settings, highlighting the superior scalability and convergence speed of \textbf{\rosecdl{}}. The runtime plots show the evolution of test loss over time. The third subplot reports the dictionary recovery score at convergence for 2D data. As \alphacsc cannot be applied to 2D data, only results for Sporco, DeepCDL, and \textbf{\rosecdl{}} are shown.
    }
    \label{fig:RoseCDL_efficiency}
\end{figure*}

\paragraph{Threshold selection.}
From the distribution of patch reconstruction errors $(\beps_{\tau})_{\tau} = (F(D, Z; \bx_{\tau}))_{\tau}$, one needs to compute a threshold $\beta$ which separates the outlier patches in $\bx$ from the normal ones.
The goal is to detect extreme points in the distribution, which are too large compared to the population of patch errors.
The outlier detection literature provides three main methods relevant in this case:
\begin{enumerate}
\item The \textbf{method of quantiles}, where $\beta = Q_{\beps, (1-\alpha)}$ is the quantile of order $(1 - \alpha)$ of the set $\beps$,
\item The \textbf{$z$-score method}~\citep{iglewicz1993volume}, where each error $\beps_{\tau}$ has an associated score $z_{\tau} = (\beps_{\tau} - \mu_{\beps}) / \sigma_{\beps}$, with $\mu_{\beps}$ and $\sigma_{\beps}$ denoting respectively the mean and standard deviation of the distribution of errors, and where outliers are defined as observations such that $\abs{z_{\tau}} > \alpha$, generally $\alpha = 2 \text{ or } 3$, thus having $\beta = \mu_{\beps} + \alpha\ \sigma_{\beps}$,
\item The \textbf{modified $z$-score} (MAD;~\citealt{iglewicz1993volume}), where, similarly as the $z$-score, each error point is associated with a score based on the median, and the resulting threshold is $\beta = \mathrm{Med}_\beps + (\alpha\ \mathrm{Mad}_\beps)/0.6745$, with $\mathrm{Mad}_\beps = \mathrm{Med}\pars{\abs{\beps - \mathrm{Med}_\beps }}$, where \text{Med} denotes the median operator, and generally $\alpha = 3.5$.
Here, $\alpha = 3.5$ is not a tunable hyperparameter but a standard choice in robust statistics~\citep{iglewicz1993volume}. The threshold is computed adaptively from the empirical distribution of reconstruction errors, making it data-driven, statistically grounded, and independent of any contamination-rate assumptions.
It is not inflated by extreme reconstruction errors caused by outliers, which prevents it from being excessively high, which would reduce detection sensitivity.

\end{enumerate}
These methods are initially bilateral, but only the upper bound is considered in this work because we aim to detect outliers with large reconstruction errors.
\cref{fig:apnea_outliers_detection_viz} illustrates the outlier detection method on real ECG data, with the outliers mask computed with the reconstruction error and the threshold. 

\paragraph{Role of the outlier mask for rare-event detection.}
The inline outlier detection module is a central component of our algorithm.
It enables the computation of an outlier mask during training, serving two key purposes:
(1) it excludes identified outliers from the loss function, thereby improving the robustness of dictionary learning;
(2) it enables the unsupervised detection of rare events in the signal by interpreting the outlier mask as a detection map.
Indeed, provided a signal $\bx$ in the form of Eq.~\eqref{eqn:three_regimes}, we have shown that \rosecdl{} is able to extract the contribution of the common pattern $\bd_a$.
Consequently, the corresponding outlier mask can be used to select the residual signal $\bx' = \bd_b * \bz_b + \bn$.
Running another instance of \rosecdl{} on $\bx'$ then allows one to recover the pattern $\bd_b$ and localize its occurrences $\bz_b$.
In multivariate time series, \rosecdl{} computes reconstruction errors on multichannel patches, so even small deviations on individual channels can accumulate into a global deviation that exceeds the computed threshold. In practice, multivariate anomalies are detected whenever their combined effect sufficiently alters the local motif.

\section{NUMERICAL EXPERIMENTS}
\label{sec:numerical_experiments}

In this section, we present numerical experiments that demonstrate the scalability and robustness of \rosecdl, as well as its ability to detect anomalies in temporal (1D) and spatial (2D) signals, with an arbitrary number of channels, on simulated and real-world data.
We implemented the algorithm using the \texttt{Pytorch} framework~\citep{Paszke2017}.
\footnote{The code is available at \url{https://github.com/tomMoral/RoseCDL}.}
The experiments fall into three categories, each targeting a different aspect: scalability, robustness to outliers, and anomaly detection.
% For conciseness, we report one representative experiment per category here, and defer additional results to Appendix~\ref{sec:add_exp}.

\subsection{Scalability}
We performed a comprehensive comparison between \rosecdl{} and three state-of-the-art CDL methods: \alphacsc~\citep{dupre2018multivariate}, Sporco~\citep{wohlberg-2017-sporco}, and DeepCDL, a variant of~\citet{tolooshams2022stable}.
DeepCDL represents the unrolled variant of \rosecdl{} where windows are not stochastically sampled and gradients account for the Jacobian of the sparse code computed through backpropagation.
In contrast, \rosecdl{} employs an alternating minimization scheme as explained in \cref{sec:methods}.

We evaluate each method's computational cost and runtime on two large-scale datasets.
The one-dimensional (1D) experiment is performed on 20 synthetic multivariate signals containing 50{,}000 time samples with two channels (see \cref{subsec:cdl_data_simulation} for details).
The two-dimensional (2D) experiment is conducted on a semi-synthetic dataset of images of $2000 \times 2000$ pixels.
The cost is evaluated as the value of the objective function $F(D, Z^{\star}(D); \bx)$ at each iteration on a separate test set, with $Z^{\star}$ computed to convergence, assessing the capacity of the solver to minimize Eq.~\eqref{eqn:expected_cdl}.
To evaluate dictionary recovery, we use the metric proposed by~\citet{moreau2020dicodile}, which aligns true and estimated atoms by convolutional cosine similarity (see~\cref{sec:dict_evaluation}).

\begin{figure*}[!ht]
    \centering
    \begin{subfigure}[t]{0.69\linewidth} % <-- the [t] forces top alignment
        \centering
        \includegraphics{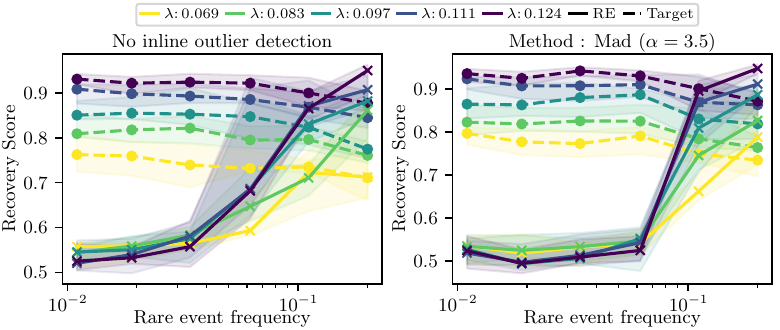}
        \label{fig:wincdl_2d_data}
    \end{subfigure}%
    \hfill
    \begin{subfigure}[t]{0.26\linewidth} % <-- also [t]
        \centering
        \includegraphics{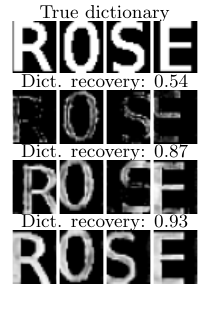}
    \end{subfigure}
    \caption{
    Comparison between \rosecdl{} without inline outlier detection (left panel) and with it (middle panel).
    Evolution of median recovery scores for \rosecdl{} on the \texttt{ROSE+Z} data set across 20 independent trials.
    Shaded regions represent the range between 25th and 75th percentiles.
    Right panel: illustration of recovered atoms for different recovery scores.
    } 
    \label{fig:wincdl_2d_data}
\end{figure*}

The empirical results shown in the left and middle panels of \cref{fig:RoseCDL_efficiency} demonstrate that \rosecdl{} exhibits superior scalability attributable to two architectural advantages shared with our implementation of DeepCDL: \textbf{(i)} GPU-optimized training yields substantial speedups when properly leveraged \textbf{(ii)} Fast Fourier Transform-based convolutions significantly reduce the computation time for large kernels compared to standard spatial convolutions.
The distinction from DeepCDL arises primarily from optimization strategy differences:
DeepCDL's unrolled architecture requires gradient computation over substantially larger parameter sets, increasing computational overhead compared to \rosecdl{}'s alternating minimization approach.

The right panel of \cref{fig:RoseCDL_efficiency} demonstrates that \rosecdl{} matches the dictionary recovery performance of DeepCDL and substantially better than Sporco, validating both computational efficiency and solution quality.

It is important to note that \alphacsc\ does not support two-dimensional data and was excluded from image-based experiments.
Despite this limitation, \rosecdl{} achieves comparable or superior test costs across all evaluated configurations.

To ensure that the results above are not artifacts of parameter choices, we evaluated \rosecdl{} under different setups and compared it to \alphacsc, assessing both computation acceleration and potential trade-offs in solution quality.
Varying the window size from 10 to 100 times the atom length produced only minor differences, with performance more strongly influenced by the choice of optimizer.
Using Adam resulted in a 7--8$\times$~speedup while keeping the error within +4\% of \alphacsc's, whereas SLS achieved smaller speedups (4.5--5.5$\times$) but reduced the error compared to \alphacsc.
When varying the signal length, \rosecdl{} exhibited sublinear scaling and remained usable on signal lengths where \alphacsc became prohibitive.
These experiments are detailed in Appendix~\ref{subsec:add_exp_scalability}.
\subsection{Robustness to outliers}

\begin{table*}[ht]
\centering
\resizebox{\textwidth}{!}{
\begin{tabular}{c@{}c l l c c c c c c c} 
\toprule
& & \textbf{Dim.} & \textbf{Dataset} 
& \textbf{\rosecdl{}} & \textbf{AB} & \textbf{DAGMM} & \textbf{MP} & \textbf{AT} 
& \textbf{TimesFM} & \textbf{TimesNet} \\
\midrule
\multirow{3}{*}{\rotatebox{90}{Pattern}} & \ldelim\{{3}{1pt} & U & ECG       & \textbf{0.534} & 0.099 & 0.077 & \underline{0.163} & 0.095 & 0.150 & 0.151 \\
& & U & SVDB      & \textbf{0.431} & 0.217 & -- & \underline{0.224} & 0.106 & 0.119 & 0.121 \\
& & M & Simulated & \textbf{0.986} & \underline{0.704} & 0.644 & --    & 0.624 & 0.652 & 0.590 \\
\midrule
\multirow{5}{*}{\rotatebox{90}{No Pattern}} & \ldelim\{{5}{1pt} & U & Daphnet  & 0.111 & \underline{0.186} & -- & 0.101 & 0.060 & 0.048 & \textbf{0.289} \\
& & U & MITDB     & 0.155 & \underline{0.186} & -- & \textbf{0.250} & 0.135 & 0.136 & 0.152 \\
& & U & Dodgers   & \textbf{0.238} & 0.078 & 0.102 & \underline{0.199} & 0.140 & 0.087 & 0.126 \\
& & M & MSL       & 0.136 & \textbf{0.171} & 0.088 & --    & 0.053 & \underline{0.136} & 0.111 \\
& & M & SMAP      & 0.111 & 0.136 & -- & --    & \textbf{0.564} & 0.107 & \underline{0.140} \\
\midrule
\multicolumn{4}{l}{\textbf{Avg. Runtime (s)}}
& \textbf{23.47} & 2594.71 & 23221.20 & \underline{687.12} & 11715.00 & 3986.07 & 1916.37 \\
\midrule
\multicolumn{4}{l}{\textbf{Avg. Pattern}}
& \textbf{0.65 $\pm$ 0.30} & 0.34 $\pm$ 0.32 & \underline{0.36 $\pm$ 0.40} & 0.19 $\pm$ 0.04 & 0.28 $\pm$ 0.30
& 0.31 $\pm$ 0.30 & 0.29 $\pm$ 0.26 \\
\multicolumn{4}{l}{\textbf{Avg. No Pattern}} 
& 0.15 $\pm$ 0.05 & 0.15 $\pm$ 0.05 & 0.09 $\pm$ 0.01 & \underline{0.18 $\pm$ 0.08} & \textbf{0.19 $\pm$ 0.21} 
& 0.10 $\pm$ 0.04 & 0.16 $\pm$ 0.07 \\
\bottomrule
\end{tabular}
}
\vspace{6pt}
\caption{
Comparison of AUC-PR of \rosecdl{} with state-of-the-art anomaly detection methods on both pattern (P) and non-pattern (NP) datasets, spanning univariate (U) and multivariate (M) settings. \rosecdl{} achieves the strongest performance on pattern anomaly datasets, its target regime, while remaining competitive on non-pattern data. Notably, \rosecdl{} delivers these results with a runtime more than two orders of magnitude faster than all deep baselines, and unlike Matrix Profile (MP), it scales seamlessly to multivariate inputs. Bold values indicate the best score in each row, underlined the second best. (AB : AnomalyBERT, AT : AnomalyTransformer).
}
\label{tab:sota_comparison}
\end{table*}

In this subsection, we study the role of the inline outlier detection module in enabling \rosecdl{} to reconstruct common patterns in the presence of anomalies and outlier patterns.
First, an example of \rosecdl{} mechanism on the Physionet Apnea-ECG~\citep{penzel2000apnea} is shown in \cref{fig:apnea_outliers_detection_viz} highlighting parts of the data that are discarded during training when using the inline outlier module as well as the outlier mask produced.
This shows that the reconstruction error highlights part of the signal that significantly differ from the rest, allowing to identify anomalies in the data.
This also stabilizes the learning, as shown in \cref{subsec:app_physionet} which presents a side-by-side visualization of atoms learned with and without inline outlier detection, demonstrating the clear benefits of this module.

To better quantify this effect, we constructed the \texttt{ROSE+Z} dataset, a semi-synthetic corpus of images generated from 5,000 characters drawn from four letters (\texttt{R}, \texttt{O}, \texttt{S}, and \texttt{E}) together with spaces, mimicking text-like documents.
A small proportion of the letter \texttt{Z} was added to introduce rare events.
Using this dataset, we assess how \rosecdl{} recovers the underlying patterns under varying levels of contamination.
For this experiment, we implemented the inline outlier detection module with the MAD method described in \cref{subsec:outlier_detection}, as preliminary comparisons against the quantile and z-score methods consistently favored MAD (see \cref{subsec:appx_choice_outlier_method}).
We then evaluated \rosecdl{} with and without inline outlier detection across 20 independent trials on the \texttt{ROSE+Z} dataset.
\cref{fig:wincdl_2d_data} shows the recovery score for both the common letters \texttt{ROSE} and the outlier letter \texttt{Z} depending on the outlier frequency.

The inline outlier detection stabilizes training, improves the quality of the learned dictionary, and reduces variability across trials by preventing rare events from contaminating the learning process.
While without inline outlier detection, the recovery slowly decreases, as the outlier gets mixed in with the common patterns, the inline mechanism allows for a sharper transition once the outlier frequency becomes too large and cannot be discarded anymore.
A key parameter of RoseCDL is the regularization coefficient $\lambda$, introduced in \cref{sec:methods}, which controls the sparsity of the learned activation vector.
We express $\lambda$ as a fraction of the maximum regularization value, $\lambda_{max}$, corresponding to the smallest regularization value for which the activation vector is entirely zero.
Our insights indicate that the best dictionary recovery is achieved when setting $\lambda$ = 0.1$\lambda_{max}$, effectively balancing sparsity and reconstruction accuracy. Qualitative and quantitative results are given in \cref{subsec:appx_lambda}.
The inline outlier detection also improves the algorithm's robustness with respect to this choice of the regularization parameter~$\lambda$.

\subsection{Anomaly detection}

We evaluated the anomaly detection capability of \rosecdl{} on diverse datasets from the TSB-UAD and TSB-AD benchmarks \citep{paparrizos2022tsb,liuelephant2024}, covering both univariate and multivariate cases.
Since \rosecdl{} targets \emph{pattern anomalies}, we included datasets both well- and poorly-aligned with this inductive bias to assess its strengths and limitations.

For baselines, we compared against representative unsupervised methods across two major families identified by~\citet{darban2024deep}: forecasting-based and reconstruction-based.
Specifically, we selected competitive examples such as the Anomaly Transformer \citep[AT;][]{xu2022anomaly} and TimesNet \citep{wutimesnet}, which respectively exemplify reconstruction-based and forecasting-based approaches.
To further broaden the comparison, we included the classical Matrix Profile method \citep[MP;][]{yehtimeseries2016}, representing a strong shallow baseline, as well as the recently proposed foundation model TimesFM \citep{das2024decoder}, which extends beyond the scope of recent surveys such as~\citet{Ruff2021} and~\citet{darban2024deep}.

As summarized in \cref{tab:sota_comparison}, \rosecdl{} consistently achieves competitive or superior performance in terms of AUC-\textcolor{blue}{PR} on datasets with pattern anomalies.
This validates its central design principle of leveraging local structures.
On datasets lacking strong patterns (e.g., Daphnet, Dodgers, MSL, SMAP) or where anomalies arise outside the pattern dimension (e.g., MITDB), performance drops but remains comparable to baselines, highlighting both the scope and the boundaries of the proposed method.
More precisely, anomalies spanning several atom lengths can still be detected, but anomalies affecting long-range temporal dynamics or global drifts without disrupting local shape are less aligned with CDL assumptions, which explains the performance gap on datasets like MITDB where anomalies primarily affect heartbeat timing rather than morphology.

Examples of both missed and detected anomalies on various datasets are given in \cref{subsec:anom_examples}.
Beyond accuracy, \rosecdl{} is over two orders of magnitude faster than competing methods, enabling scalability to large collections of time series.

These results demonstrate that \rosecdl{} can extract meaningful patterns from corrupted real-world data without explicit preprocessing or prior knowledge of the outlier proportion.
This is made possible by the inline outlier detection module, which enhances both training stability and robustness to hyperparameter choices.

\section{CONCLUSION}\label{sec:conclusion}
\looseness=-1
In this study, we introduce \rosecdl{}, a robust and scalable approach to Convolutional Dictionary Learning (CDL).
While CDL has traditionally focused on reconstructing signals, we reframe it as estimating the distribution of local patches.
This approach supports the two main features of \rosecdl{}: stochastic windowing for scalability and inline outlier detection for robustness to anomalies.
Together, these components make \rosecdl{} suitable not only for large-scale pattern discovery but also for unsupervised anomaly and rare-event detection.
Concretely, the method first models the common patterns in a signal, removes them, and then learns a dictionary on the residuals.
This yields a robust representation and an outlier mask, while also uncovering the distinctive patterns that occur specifically within anomalous regions.

A central concept of this paper is reframing CDL as a tool to characterize the distribution of patterns in the signal.
Beyond improving scalability and robustness, this perspective also provides a principled foundation for anomaly and rare-event detection:
local reconstruction errors naturally reflect deviations from the learned patch distribution.

\paragraph{Limitations.}
Like all CDL-based methods, our approach relies on a non-convex objective.
Nevertheless, the optimization scheme builds on well-established principles, for which convergence to meaningful local minima has been observed and studied in prior work~\citep{gribonval_sparse_spurious,malezieux2021understanding,tolooshams2022stable}.
Our trimming strategy intentionally excludes poorly reconstructed patches during training, which introduces bias but is consistent with our design goal of separating atypical regions rather than fitting them.
While this mechanism confers robustness to outliers, it is conservative by nature: many segments may be flagged as atypical.
Consequently, \rosecdl{} is better suited for robust representation learning and rare event detection than for high-precision anomaly scoring.

% XXX: todo add acknowledgement
\section*{Acknowledgement}

This project was supported by the French National Research Agency (ANR) through the BenchArk project (ANR-24-IAS2-0003) and EBUL project (ANR-
23-CE23-0001).
Mansour Benbakoura was supported from a national grant attributed to the ExaDoST project of the NumPEx PEPR program, under the reference ANR-22-EXNU-0004.
This work was performed using HPC resources from GENCI–IDRIS (Grant 2025-AD011015308R1). 

\bibliographystyle{abbrvnat}

\bibliography{references}

%%%%%%%%%%%%%%%%%%%%%%%%%%%%%%%%%%%%%%%%%%%%%%%%%%%%%%%%%%%%
\section*{Checklist}

% %%% BEGIN INSTRUCTIONS %%%
% The checklist follows the references. For each question, choose your answer from the three possible options: Yes, No, Not Applicable.  You are encouraged to include a justification to your answer, either by referencing the appropriate section of your paper or providing a brief inline description (1-2 sentences). 
% Please do not modify the questions.  Note that the Checklist section does not count towards the page limit. Not including the checklist in the first submission won't result in desk rejection, although in such case we will ask you to upload it during the author response period and include it in camera ready (if accepted).

% \textbf{In your paper, please delete this instructions block and only keep the Checklist section heading above along with the questions/answers below.}
% %%% END INSTRUCTIONS %%%

\begin{enumerate}

  \item For all models and algorithms presented, check if you include:
  \begin{enumerate}
    \item A clear description of the mathematical setting, assumptions, algorithm, and/or model. Yes
    \item An analysis of the properties and complexity (time, space, sample size) of any algorithm. Yes
    \item (Optional) Anonymized source code, with specification of all dependencies, including external libraries. Yes
  \end{enumerate}

  \item For any theoretical claim, check if you include:
  \begin{enumerate}
    \item Statements of the full set of assumptions of all theoretical results. Yes
    \item Complete proofs of all theoretical results. Yes
    \item Clear explanations of any assumptions. Yes
  \end{enumerate}

  \item For all figures and tables that present empirical results, check if you include:
  \begin{enumerate}
    \item The code, data, and instructions needed to reproduce the main experimental results (either in the supplemental material or as a URL). Yes
    \item All the training details (e.g., data splits, hyperparameters, how they were chosen). Yes
    \item A clear definition of the specific measure or statistics and error bars (e.g., with respect to the random seed after running experiments multiple times). Yes
    \item A description of the computing infrastructure used. (e.g., type of GPUs, internal cluster, or cloud provider). Yes
  \end{enumerate}

  \item If you are using existing assets (e.g., code, data, models) or curating/releasing new assets, check if you include:
  \begin{enumerate}
    \item Citations of the creator If your work uses existing assets. Yes
    \item The license information of the assets, if applicable. Not Applicable
    \item New assets either in the supplemental material or as a URL, if applicable. Yes
    \item Information about consent from data providers/curators. Yes
    \item Discussion of sensible content if applicable, e.g., personally identifiable information or offensive content. Not Applicable
  \end{enumerate}

  \item If you used crowdsourcing or conducted research with human subjects, check if you include:
  \begin{enumerate}
    \item The full text of instructions given to participants and screenshots. Not Applicable
    \item Descriptions of potential participant risks, with links to Institutional Review Board (IRB) approvals if applicable. Not Applicable
    \item The estimated hourly wage paid to participants and the total amount spent on participant compensation. Not Applicable
  \end{enumerate}

\end{enumerate}

%%%%%%%%%%%%%%%%%%%%%%%%%%%%%%%%%%%%%%%%%%%%%%%%%%%%%%%%%%%%%%%%%%%%%%%%%%%%%%%
%%%%%%%%%%%%%%%%%%%%%%%%%%%%%%%%%%%%%%%%%%%%%%%%%%%%%%%%%%%%%%%%%%%%%%%%%%%%%%%
% APPENDIX
%%%%%%%%%%%%%%%%%%%%%%%%%%%%%%%%%%%%%%%%%%%%%%%%%%%%%%%%%%%%%%%%%%%%%%%%%%%%%%%
%%%%%%%%%%%%%%%%%%%%%%%%%%%%%%%%%%%%%%%%%%%%%%%%%%%%%%%%%%%%%%%%%%%%%%%%%%%%%%%
\newpage
\appendix
\onecolumn
\counterwithin{figure}{section}
\def\thefigure{\thesection.\arabic{figure}}
\counterwithin{table}{section}
\def\thetable{\thesection.\arabic{table}}

\section{ANALYTICAL STUDY}\label{sec:study}

\stability*

\begin{proof}
We consider a dictionary $D \in \R^{1\times L}$ with a single atom $\bd$.
As we consider signals composed of patterns with no overlap, we can separate each segment and we have a population of signals $X = z d_i + \epsilon$, with $z \in \R$, $\epsilon \sim \mathcal N(0, \sigma^2 I)$ and $d_i = \bd_a$ with probability $1-\rho$ and $\bd_b$ with probability $\rho$, with $\|\bd_a\|_2 = \|\bd_b\|_2 = 1$.
We consider all atoms $\bd, \bd_a, \bd_b$ to be unit norm.
Wlog, we can consider $z = 1$, as this amounts to rescaling the value of $\lambda_{max}$, and we consider that $c_a=\bd^\top \bd_a$ and  $c_b = \bd^\top \bd_b$ are positive, as we can consider $-\bd$ otherwise.
We also consider that the noise level is small enough such that $\sigma^2 < c_j$.

This model is a simplified model in which we have a population of signals where we want to identify the pattern of an event $\bd_a$ from the pattern of a rare event $\bd_b$.

In this setting, if we further have that the auto-correlation of $\bd$ with $\bd_a$ and $\bd_b$ is maximal when they are aligned, then the sparse coding of a signal $X$ can be computed with the following formula:
\begin{equation}
z^*(X, \bd) = \begin{cases}
    0 & \text{if}\quad c + \epsilon^\top \bd \le \lambda\\
    c + \epsilon^\top \bd - \lambda & \text{otherwise}
\end{cases}
\end{equation}
with $c = \bd^\top d_i$, which has value $c_a$ with probability $1-\rho$ and $c_b$ otherwise.

We can compute the loss value for this $z^*(X, \bd)$ for $X$ where $z^*$ is non-zero:
\begin{align}
    F(\bd, z^*; X) & = \frac12 \|X - z^* \bd\|_2^2 + \lambda \|z^*\|_1\\
            & = \frac12 \left(\|X\|_2^2 -2 (c + \epsilon^\top \bd - \lambda) (c + \epsilon^\top \bd) + \|(c + \epsilon^\top \bd - \lambda)\bd\|_2^2\right) + \lambda |c + \epsilon^\top \bd - \lambda|\\
            & = \frac12 \left(\|X\|_2^2 -2 (c + \epsilon^\top \bd - \lambda) (c + \epsilon^\top \bd) + (c + \epsilon^\top \bd - \lambda)^2 + 2 \lambda (c + \epsilon^\top \bd - \lambda)\right) \\
            & = \frac12 \left(\|X\|_2^2 -2 (c + \epsilon^\top \bd - \lambda) (c + \epsilon^\top \bd - \lambda) + (c + \epsilon^\top \bd - \lambda)^2\right) \\
            & = \frac12 \left(\|X\|_2^2 -  (c + \epsilon^\top \bd - \lambda)^2\right) \\
            & = \frac12 \left(\|d_i\|_2^2 -  (c - \lambda)^2 + \|\epsilon\|_2^2 - (\epsilon^\top \bd)^2 -2(1 - (c-\lambda))\epsilon^\top \bd\right) \\
\end{align}
Taking the expectation over the noise yields:
\begin{equation}
    \bbE_{\epsilon} \left[F(\bd, z^*; X)\right] = \frac12 \left(1 -  (c - \lambda)^2 + (L-1)\sigma^2\right)
    \label{eq:optimal_loss}
\end{equation}For $c$ between $\lambda$ and $1$, this function is decreasing in $c$, meaning that for two samples constructed with $\bd_a$ and $\bd_b$, if the correlation $c_0 = \bd^\top \bd_a$ is larger than the correlation $c_b = \bd^\top \bd_b$, then the reconstruction loss for sample $0$ is smaller in expectation than the reconstruction loss for a sample $1$.

We can also compute the gradient of this function with respect to $\bd$.
Note that with the KKT condition defining $z^*$, we have that the $\nabla_z F(\bd, z*; X) = 0$, and thus we do not need to compute the Jacobian of $z^*$ when computing the derivative of $F$ with respect to $\bd$.
The gradient reads:
\begin{align}
    \nabla_{\bd} F(\bd, z^*; X) & = z^*(z^* \bd - X)\\
            & = (z^*)^2 \bd - z^*X\\
            & = (c + \epsilon^\top \bd - \lambda)^2 \bd - (c + \epsilon^\top \bd - \lambda)(d_i + \epsilon)
\end{align}
Taking the expectation over the noise yields:
\begin{align}
    \bbE_{\epsilon} \left[\nabla_{\bd} F(\bd, z^*; X)\right] &= ((c - \lambda)^2 + \sigma^2) \bd - (c - \lambda)d_i + \underbrace{\bbE_{\epsilon}\left[\epsilon^\top \bd\epsilon\right]}_{\sigma^2\bd}\\
    &= ((c - \lambda)^2 + 2\sigma^2) \bd - (c - \lambda)d_i
\end{align}
This yields
$$
\mathbb{E}[\nabla_\bd F(\bd, z^*; X)] 
= \left((1-\rho)(c_a - \lambda)^2 + \rho(c_b - \lambda)^2 + 2\sigma^2\right)\bd
- (1-\rho)(c_a - \lambda)\bd_a - \rho(c_b - \lambda)\bd_b
$$

In the noiseless case, if $\bd = \bd_a$, and $\lambda \le c_b = (\bd_b)^\top \bd_a < 1$, with the classical algorithm, the expected gradient reads
\begin{align}
    \bbE_{X} \left[\nabla_{\bd} F(\bd_a, z^*; X)\right] & = -(1-\rho)\lambda(1-\lambda)\bd_a + \rho((c_a -  \lambda)^2 \bd_a - (c_b - \lambda) \bd_b)\\
            &= \left( \rho(c_a -  \lambda)^2 -(1-\rho)\lambda(1-\lambda)\right)\bd_a - \rho(c_b - \lambda) \bd_b
\end{align}
This gradient is not colinear with $\bd_a$, showing that $\bd_a$ is not a fixed point of the projected gradient descent algorithm in this context.
Even in a noiseless and very simple setting, the $\bd_a$ is not a solution of the Classical CDL algorithm.

In contrast, when using the least trimmed square procedure with a trimming threshold rejecting a proportion $\rho$ of the samples, we can show that $\bd_a$ is a fixed point in the noiseless setting.
As seen in \eqref{eq:optimal_loss}, the loss for samples $X$ associated with $\bd_b$ is smaller than the loss for samples associated with $\bd_a$, and therefore rejecting $\rho$ samples from the gradient computation leads to:
\begin{equation}
    \bbE_{X} \left[\nabla_{\bd} F(\bd_a, z^*; X)\right] = -(1-\rho)\lambda(1-\lambda)\bd_a
\end{equation}
as the gradient is colinear with $\bd_a$, thus $\bd_a$ is a fixed point of the projected gradient descent and of the learning procedure.
\end{proof}

\clearpage
\section{ADDITIONAL EXPERIMENTS}
\label{sec:add_exp}

In this section, we present two numerical experiments intended to characterize the behavior of RoseCDL.
We first discuss the influence of the regularization parameter~$\lambda$ on the ability of RoseCDL to learn the atoms.
Then, we highlight the importance of the inline outlier detection module in making the CDL robust to outliers.

\subsection{Scalability}
\label{subsec:add_exp_scalability}

To thoroughly assess scalability, we conduct additional experiments varying both window sizes and signal lengths.

\subsubsection{Window size}

For window size analysis on 1D signals with $T=100{,}000$ and $\lambda=0.8$, we evaluate both Adam and SLS optimizers across window sizes ranging from $10L$ to $100L$. To ensure full GPU utilization, we maintain a constant product window size $\times$ batch size. As shown in \cref{tab:window_scaling}, RoseCDL consistently achieves validation losses within $4\%$ of \alphacsc\ performance across all configurations, while maintaining runtimes of $12$--$22\%$ relative to \alphacsc, corresponding to approximately $5\times$ speedup.

\begin{table*}[htbp]
\centering
\begin{tabular}{lcccccccccc}
\toprule
\textbf{Optimizer} & \multicolumn{4}{c}{\textbf{Adam}} & & \multicolumn{4}{c}{\textbf{SLS}} \\
\cmidrule(lr){2-5} \cmidrule(lr){7-10}
\textbf{Window size} 
& $10L$ & $20L$ & $50L$ & $100L$ & & $10L$ & $20L$ & $50L$ & $100L$ \\
\midrule
\textbf{Validation loss} 
& $+3.0\%$ & $+3.8\%$ & $+2.7\%$ & $+2.7\%$ & & $-0.2\%$ & $-0.2\%$ & $-0.2\%$ & $\textbf{-0.4\%}$ \\
\textbf{Runtime} 
& $15.4\%$ & $\textbf{12.3\%}$ & $12.5\%$ & $12.7\%$ & & $21.7\%$ & $21.2\%$ & $16.3\%$ & $17.6\%$ \\
\bottomrule
\end{tabular}

% \vspace{6pt}
\caption{
Comparison of Adam and SLS with different window sizes.
Validation loss expressed relatively to \alphacsc's.
Runtime expressed in proportion to \alphacsc's.
}
\label{tab:window_scaling}
\end{table*}

Notably, the SLS optimizer demonstrates superior convergence properties with validation losses within $0.4\%$ of \alphacsc, though with slightly increased runtime compared to Adam. The results validate that border effects do not hinder convergence with our stochastic windowing approach, even for small window sizes (as low as $10L$).

\subsubsection{Signal length}

For signal length scalability, we evaluate performance on signals ranging from $10$k to $1$M time samples. As detailed in \cref{tab:signal_scaling}, RoseCDL demonstrates sublinear scaling, with runtimes growing from $15.2$\,s ($10$k samples) to $202.8$\,s ($1$M samples), while \alphacsc\ becomes computationally prohibitive beyond $100$k samples. This superior scalability enables RoseCDL to process signals substantially larger than existing full-signal sparse coding methods.

\begin{table*}[htbp]
\centering
\begin{tabular}{lccccc}
\toprule
$T$ & 10k & 30k & 100k & 300k & 1M \\
\midrule
Runtime RoseCDL (s)    & 15.2 & 16.4 & 32.6 & 68.0 & 202.8 \\
Runtime \alphacsc\ (s) & 67.5 & 102.8 & 198.5 & N/A  & N/A  \\
\bottomrule
\end{tabular}

% \vspace{6pt}
\caption{Signal length scaling: runtime comparison between RoseCDL and \alphacsc\ for varying signal sizes $T$.}
\label{tab:signal_scaling}
\end{table*}

\subsection{Choice of the outlier detection method}
\label{subsec:appx_choice_outlier_method}

To evaluate the performance of RoseCDL on 2D data, we constructed a semi-synthetic dataset of images by generating 5000 characters sampled from a set of four letters (\texttt{R}, \texttt{O}, \texttt{S}, and \texttt{E}) along with spaces.
These images emulate text-like documents composed of words formed from the selected characters.
We added the letter \texttt{Z} in a small proportion to introduce rare events.
Experiments were conducted with a 10\% contamination rate.
Consequently, for the inline outlier detection methods, we implemented the quantile detection method with $\alpha$ value: \SI{10}{\percent}, alongside with MAD and z-score methods as described in \cref{subsec:outlier_detection}.
\begin{figure}
  \centering
\includegraphics[width=0.48\linewidth]{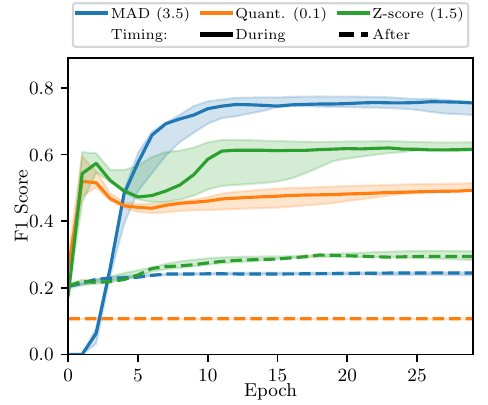} % 
  \caption{Comparison of the F1 score evolution of different methods over the epochs on rare event detection task on the (\texttt{R}, \texttt{O}, \texttt{S}, \texttt{E}, and \texttt{Z}).}
  \label{fig:all_methods_rare_event_detection}
\end{figure}
To assess the efficiency of the inline outlier module on this dataset, we compared it to RoseCDL without the inline outlier module, where we computed the threshold by computing the reconstruction error after the learning part.
We show the detection over the epochs using different outlier detection algorithms by computing the F1 score between the inline model mask for the \texttt{during} procedure and the mask after the reconstruction by a RoseCDL model without an inline outlier method for the \texttt{after} procedure in \cref{fig:all_methods_rare_event_detection}. The inline outlier detection module, particularly the MAD method, works by discarding the rare events from the dictionary learning part. These rare events are reconstructed with lower fidelity compared to the more prevalent patterns. This selective degradation results in higher reconstruction errors that are sharply localized at the rare events' positions and improves their detection. In opposition to that, in the absence of this module, rare events are still reconstructed less accurately than common patterns; however, the contrast in reconstruction quality is less distinct. This diminishes the precision of rare event detection and ultimately leads to a lower F1 score.

\subsection{Inline outlier detection on real-world data}
\label{subsec:app_physionet}

\begin{figure}[ht]
    \centering
    \includegraphics[trim=10pt 7pt 4pt 8pt, clip, width=0.4\linewidth, height=2.7in]{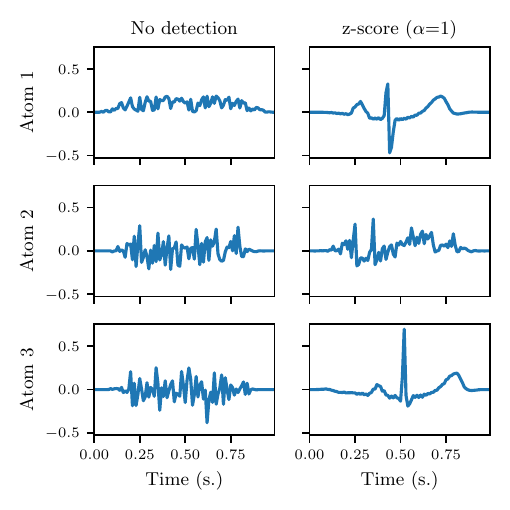}%

    \caption{Learned atoms with and without outliers detection method, on 10 bad trials of subject a02 of dataset Physionet Apnea-ECG.}%
    \label{fig:apnea_atoms}
\end{figure}

We evaluate our approach through two complementary experiments.
First, we consider the Physionet Apnea-ECG dataset~\citep{penzel2000apnea} without any preprocessing, thereby preserving the raw variability of the signals.
From a \SI{10}{\minute} ECG segment containing blocks of anomalies, we learned a three-atom dictionary (\SI{1}{\second} each) to assess robustness against corrupted data (\cref{fig:apnea_outliers_detection_viz}).
Without outlier detection, the model converges to noise-like patterns, whereas the inline anomaly detection mechanism suppresses high-variance anomalies and enables recovery of meaningful ECG atoms (\cref{fig:apnea_atoms}).
While careful parameter tuning can sometimes yield similar results, it remains unreliable and computationally costly.

\subsection{Influence of the regularization parameter}
\label{subsec:appx_lambda}

A key parameter of RoseCDL is the regularization coefficient $\lambda$, introduced in \cref{sec:methods}, which controls the sparsity of the learned activation vector.
We express $\lambda$ as a fraction of the maximum regularization value, $\lambda_{\max}$, corresponding to the smallest regularization value for which the activation vector is entirely zero.
Note that when using the trimmed objective, $\lambda_{\max}$ computation is adapted to account for the modified loss.
To analyze the effect of regularization on dictionary recovery, we compare recovery scores across different values of $\lambda$ using synthetic data.
The results, presented in \cref{fig:exp_lambda_1d}, demonstrate how varying $\lambda$ influences the quality of the recovered dictionary.
Our findings indicate that the best dictionary recovery is achieved when setting $\lambda = 0.1 \lambda_{\max}$, effectively balancing sparsity and reconstruction accuracy.

\begin{figure}[h]
    \centering
    \begin{subfigure}{0.49\linewidth}
        \centering
        \includegraphics[width=\linewidth]{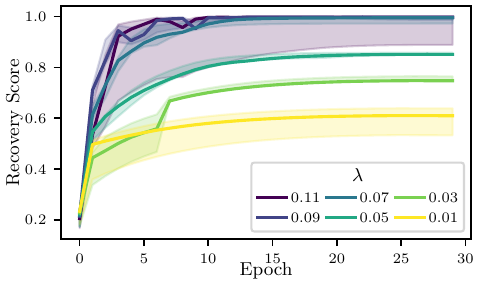}
        \label{fig:exp_lambda_1d_a}
    \end{subfigure}%
    \begin{subfigure}{0.49\linewidth}
        \centering
        \includegraphics[width=\linewidth]{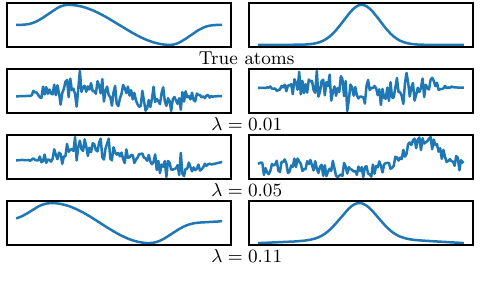} 
        \label{fig:exp_lambda_1d_b}
    \end{subfigure}
    \vspace{-6pt}
    \caption{(a) Impact of the regularization on the recovery score over the epochs.
    (b) Atoms recovered at the end of training for different regularization values.} 
    \label{fig:exp_lambda_1d}
\end{figure}

\subsection{Anomalies : Missed and detected examples}
\label{subsec:anom_examples}

\begin{figure}[h]
\centering 
\begin{subfigure}{0.49\linewidth} 
\centering 
\includegraphics[width=\linewidth]{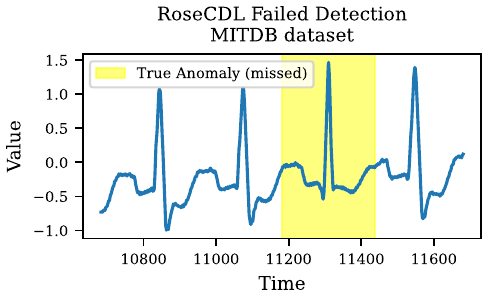} 
\end{subfigure}% 
\begin{subfigure}{0.49\linewidth} 
\centering 
\includegraphics[width=\linewidth]{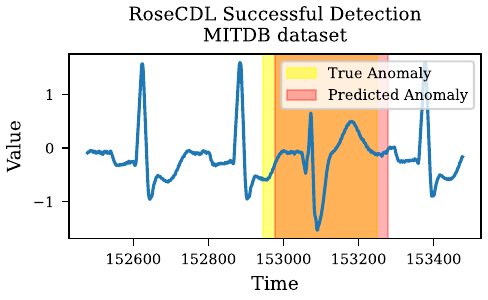} \end{subfigure} 
\vspace{-6pt} 
\caption{Examples of missed and detected anomalies on the MITDB Arythmia dataset.} 
\label{fig:mitdb_anomalies} \end{figure}

\begin{figure}[ht]
    \centering
    \begin{subfigure}{0.49\linewidth}
        \centering
        \includegraphics[width=\linewidth]{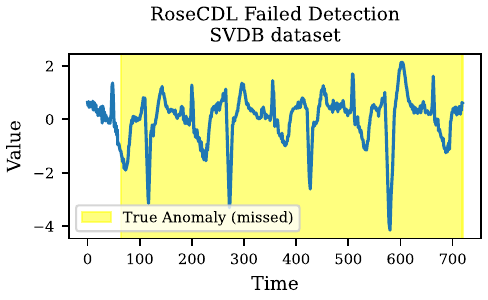}
        \includegraphics[width=\linewidth]{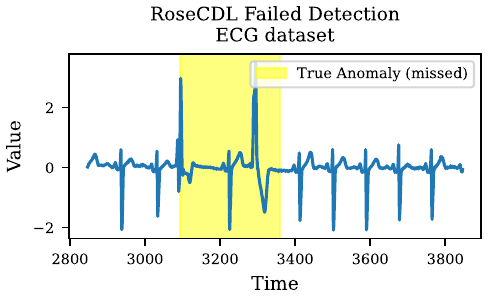}
        \includegraphics[width=\linewidth]{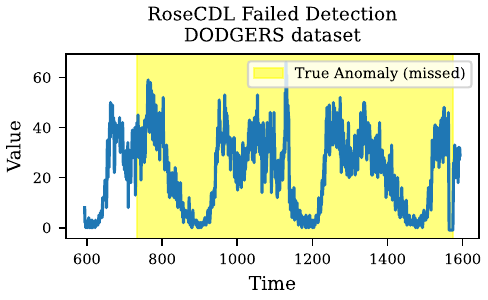}
        \includegraphics[width=\linewidth]{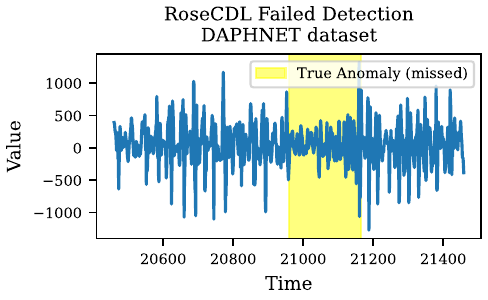}
    \end{subfigure}%
    \begin{subfigure}{0.49\linewidth}
        \centering
        \includegraphics[width=\linewidth]{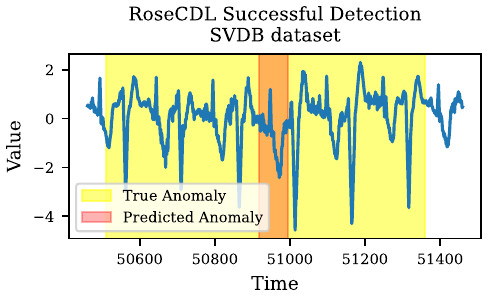}
        \includegraphics[width=\linewidth]{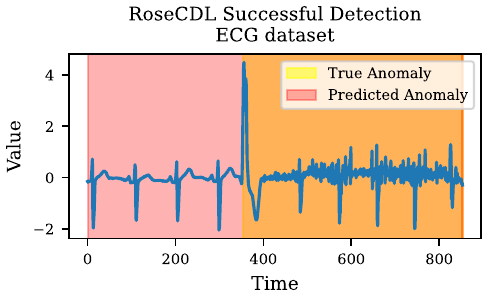}
        \includegraphics[width=\linewidth]{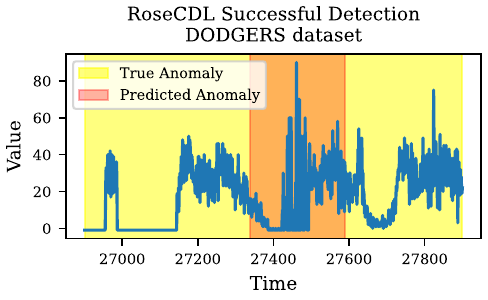}
        \includegraphics[width=\linewidth]{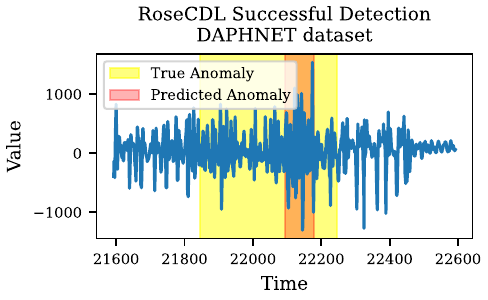}
    \end{subfigure}
    \vspace{-6pt}
    \caption{Examples of missed (left) and detected (right) anomalies across datasets: MITDB, SVDB, ECG, Dodgers, and Daphnet.}
    \label{fig:all_anomalies}
\end{figure}

\clearpage
\section{DATA SIMULATION}
\label{subsec:cdl_data_simulation}

The synthetic multivariate 1D signals $X \in \R^{P \times T}$ used in Sect.~\ref{sec:numerical_experiments} are generated from a dictionary $D \in \R^{K \times P \times L}$, a sparse activation vector $Z \in \R^{K \times (T - L + 1)}$, and a random Gaussian noise $\varepsilon \sim \Norm{0}{\sigma^2}$ as $X = D*Z + \varepsilon$. In this definition,
\begin{itemize}
    \item $P$ is the number of channels,
    \item $T$ is the length of the signal,
    \item $K$ is the number of atoms,
    \item $L$ is the length of the atoms.
\end{itemize}
In the experiments conducted in Sect.~\ref{sec:numerical_experiments}, we generated signals of length $T = 50\ 000$ with $P = 2$ channels from dictionaries with $K = 2$ atoms of length $L = 64$.
The atoms were generated from sine and gaussian waveforms, as illustrated in Fig.~\ref{fig:simulation_dict_true}.
The activations $Z$ were randomly generated sparse Dirac combs with sparsity 0.4~\% and the noise level was set to $\sigma = 0.1$.

\begin{figure}[h]
    \centering
    \includegraphics[width=\textwidth]{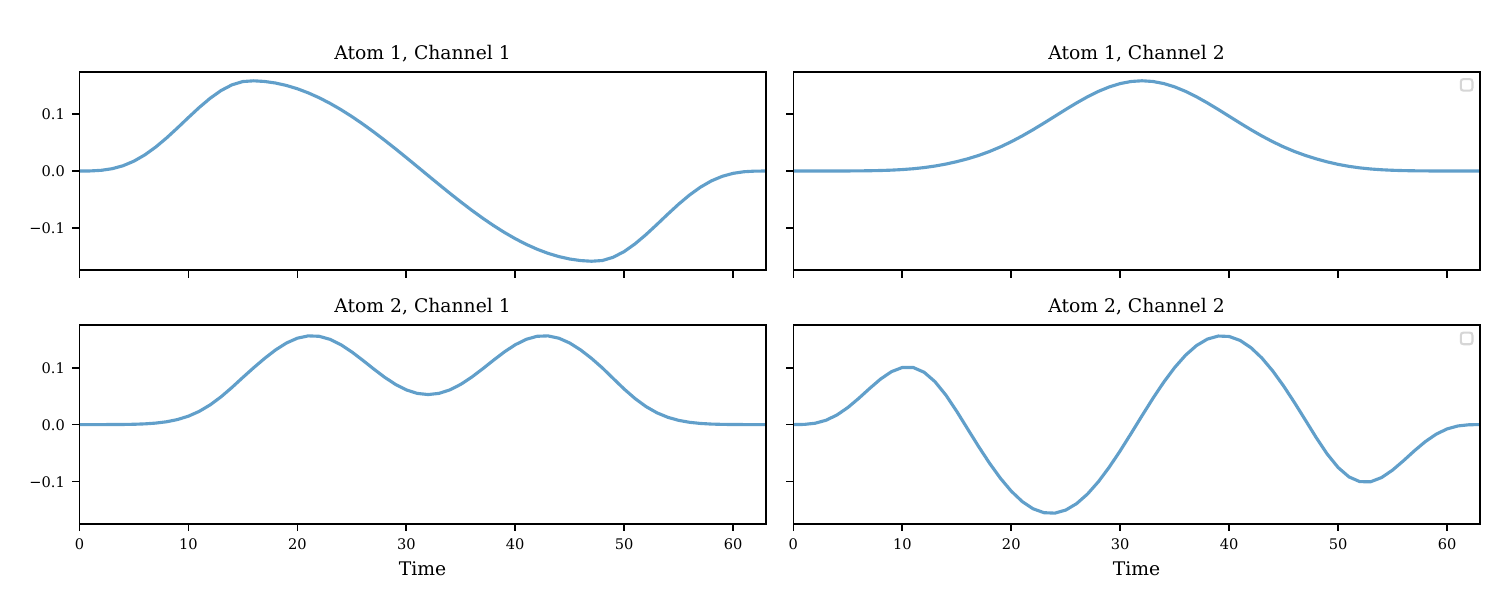}
    \caption{True dictionary in experiments on synthetic data.}
    \label{fig:simulation_dict_true}
\end{figure}
\section{DICTIONARY EVALUATION}\label{sec:dict_evaluation}
In our methodology, we evaluate the effectiveness of a learned dictionary, denoted as $\widehat{\bD}\in\R^{K'\times P\times L'}$, by comparing it against a set of true dictionary patterns, represented as $\bD\in\R^{K\times P\times L}$ and computing a \enquote{recovery score}, using the convolutional cosine similarity following optimal assignment, as defined by
~\cite{moreau2020dicodile}.
The learned dictionary and the true patterns are structured as three-dimensional arrays, where dimensions correspond to the number of atoms, channels, and atoms' duration.
The learned dictionary may differ from the true dictionary in terms of the number of atoms and the length of time atoms, typically featuring more atoms and extended durations.

The evaluation process involves a computational step known as multi-channel correlation.
In this step, each atom of the learned dictionary is systematically compared with each pattern in the true dictionary.
This comparison is carried out channel by channel, aggregating the results to capture the overall similarity between the dictionary atom and the pattern.

After performing these comparisons for all combinations of atoms and patterns, we create a matrix that represents the correlation strengths between each pair.
To objectively assess the quality of the learned dictionary, we use an optimization technique called the Hungarian algorithm.
This algorithm finds the best possible \enquote{matching} between the learned dictionary atoms and the true patterns, aiming to maximize the overall correlation.

The final score, which quantifies the performance of the learned dictionary, is derived by averaging the values of these optimal matchings.
This score is scaled between 0 and 1, where 1 represents the best possible performance.
A higher score indicates that the learned dictionary more accurately represents the true dictionary patterns, providing a measure of its quality and effectiveness in capturing the essential features of the data.

Mathematically, the recovery score between the dictionaries $\widehat{\bD}$ and $\bD$ can be expressed as follow:
\begin{equation}
    \text{score} = \frac{1}{K} \sum_{i=1}^K C_{i, j^*(i)} \enspace ,
\end{equation}
where $j^*(i), i=1,\dots,K$ denote the results of the linear sum assignment problem~\citep{crouse2016implementing}\footnote{We use the \texttt{SciPy}'s implementation.} on correlation matrix $C \coloneqq \corr{\bD}{\widehat{\bD}} \in\R^{K\times K'}$, with $\forall i\in\intervalleEntier{1}{K}$, $\forall j\in\intervalleEntier{1}{K'}$,
\begin{equation}
    C_{i,j} = \maxx{l=1, \dots, L+L'-1} \corr[\text{2D}]{D_i}{\widehat{D}_j} \bracks{l} \in \R \enspace ,
\end{equation}
where $D_i\in\R^{P\times L}$ and $\widehat{D}_j\in\R^{P\times L'}$.
The multivariate \enquote{2D} correlation between the two matrices $D$ and $\widehat{D}$ is defined as follow:
\begin{equation}
    \corr[\text{2D}]{D}{\widehat{D}} = \sum_{p=1}^P \corr[\text{1D}]{d_p}{\hat{d}_p} \in \R^{L+L'-1} \enspace ,
\end{equation}
where $d_p\in\R^L$ and $\hat{d}_p\in\R^{L'}$.
The 1D \enquote{\textit{full}} correlation between the two vectors $d$ and $\hat{d}$ is defined as follow, ${\forall t\in\intervalleEntier{1}{L+L'-1}}$:
\begin{equation}
    \corr[\text{1D}]{d}{\hat{d}}\bracks{t} = \pars{d\ast\hat{d}}\bracks{t-T+1} = \sum_{l=1}^{L} d[l]\hat{d}\bracks{l-t+T} \in \R \enspace ,
\end{equation}
where $T\coloneqq \max\pars{L,L'}$.
\section{EXPERIMENTS SETUP}
\label{experiment-setup}

\begin{table}[h!]
\centering
\begin{tabular}{p{2cm} p{4.2cm} p{1cm} p{3cm} l}
\toprule
Experiment & Data size & Number of runs & Hardware \\
\midrule
Runtime 1D   & 20 signals, 50,000 data points, 2 channels   & 50   & GPU NVIDIA A40, 30 CPUs  & \\
Runtime 2D  & $2000 \times 2000$ grayscale images   & 20   & GPU NVIDIA A40 & \\
Regularization impact 1D  & 10 signals, 30,000 data points, 1 channel  & 10  & GPU NVIDIA A40 & \\
Regularization impact 2D  & $2000 \times 2000$ grayscale images  & 10  & GPU NVIDIA A40 & \\
Inline vs After  & 2D image data  & 10  & GPU NVIDIA A40 & \\
Physionet  & 10 trials of subject a02  & 10  & GPU NVIDIA A40 and CPUs & \\

\bottomrule
\end{tabular}
\label{tab:experiment-setup}
\end{table}

\end{document}